\documentclass{article}

\usepackage{times}
\usepackage{graphicx} 
\usepackage{subfigure} 

\usepackage{natbib}

\usepackage{algorithm}
\usepackage{algorithmic}

\usepackage{subfigure}
\usepackage{stmaryrd} 

\usepackage[accepted,nohyperref]{icml2014}
\usepackage{amsmath,amssymb}
\newcommand{\BlackBox}{\rule{1.5ex}{1.5ex}}  
\newenvironment{proof}{\par\noindent{\bf Proof\ }}{\hfill\BlackBox\\[2mm]}

\newtheorem{theorem}{Theorem}
\newtheorem{lemma}[theorem]{Lemma}

\icmltitlerunning{A PAC-Bayesian Bound for Lifelong Learning}

\DeclareMathOperator*{\E}{\operatorname{\mathbf{E}}}
\DeclareMathOperator{\KL}{\operatorname{\textsl{KL}}}
\newcommand{\sign}{\operatorname{sign}}
\newcommand{\er}{\operatorname{er}}
\newcommand{\eer}{\widehat{\er}}
\newcommand{\ter}{\widetilde{\er}}
\newcommand{\tr}{\operatorname{tr}}
\newcommand{\const}{\mathit{const}}
\newcommand{\Erf}{\operatorname{erf}}

\newcommand{\argmin}{\operatorname*{argmin}}

\renewcommand{\H}{H} 

\begin{document} 

\twocolumn[
\icmltitle{A PAC-Bayesian Bound for Lifelong Learning}

\icmlauthor{Anastasia Pentina}{apentina@ist.ac.at}
\icmladdress{IST Austria (Institute of Science and Technology Austria),
			Am Campus 1, 3400 Klosterneuburg, Austria}
\icmlauthor{Christoph H. Lampert}{chl@ist.ac.at}
\icmladdress{IST Austria (Institute of Science and Technology Austria),
            Am Campus 1, 3400 Klosterneuburg, Austria}

\icmlkeywords{boring formatting information, machine learning, ICML}

\vskip 0.3in
]

\begin{abstract} 
Transfer learning has received a lot of attention in the machine 
learning community over the last years, and several effective 
algorithms have been developed. 
However, relatively little is known about their theoretical properties, 
especially in the setting of \emph{lifelong learning}, where the goal 
is to transfer information to tasks for which no data have been observed 
so far. 

In this work we study lifelong learning from a theoretical perspective.
Our main result is a PAC-Bayesian generalization bound that offers 
a unified view on existing paradigms for transfer learning, such as 
the transfer of parameters or the transfer of low-dimensional 
representations. 
We also use the bound to derive two principled lifelong learning algorithms,  
and we show that these yield results comparable with existing methods.
\end{abstract} 

\section{Introduction}
Today, many problems can be solved equally well or better by machine learning algorithms as by humans. 
However, these algorithms typically require large amount of training data to 
achieve acceptable results, whereas humans are able to learn new concepts 
from just a few examples. 
Presumably this difference comes from the fact that most machine learning systems are trained from scratch for each task at hand, whereas humans exploit context
and knowledge they acquired previously while solving other tasks. 

This observation motivates research on transfer learning: 
\emph{how can information from previously learned tasks be used for solving new tasks?}
Several scenarios of how this question can be formalized have been identified.
%
Here, we discuss some that are most relevant in the context of supervised 
learning.
As general setup, one assumes that one or more learning tasks have been 
observed, typically in form of labeled training sets. The methods then differ
in how this information is meant to be used.
In the \emph{multitask setting}~\citep{caruana1997multitask}, the goal is 
simply to perform well on all of the tasks. 
In \emph{domain adaptation}~\citep{bridle1990recnorm}, the goal is to perform  
well on a new task for which only unlabeled or very few labeled data samples 
are observed.
Finally, in \emph{lifelong learning} or \emph{learning to learn}~\citep{thrun1995lifelong}, 
the goal of the learner is to perform well on future tasks, for which so far no
data has been observed. 
In this work we focus on the third setting.

\paragraph{Lifelong Learning.}
For lifelong learning to make sense, one must assume a 
relation between the observed tasks and the future tasks.
To formalize this, \citet{Baxter} introduced the notion of a task 
environment as a set of possible tasks that might need to be 
solved at some time. The observed tasks are sampled randomly 
from the environment according to an unknown task distribution. 
In this setting, Baxter also provided the first theoretical guarantees 
by proving generalization bounds in the framework of VC theory~\citep{vapnik1998statistical}.
After this work, however, progress on the theoretical 
understanding of lifelong machine learning slowed down. 
Many algorithms for transfer learning were developed 
and found empirically to work well in many cases.
However, except for a few exception, such as~\cite{journals/ml/Maurer09, maurer13}, their theoretical justifications 
are so far not well understood. 

In this work, we aim at making progress on the theoretical justifications
of lifelong learning. In Section~\ref{sec:bound} we prove a general 
PAC-Bayesian generalization bound for lifelong learning that allows 
quantifying the relation between the expected loss on a future learning 
task to the average loss on the observed tasks.
In contrast to Baxter's results, our bound has the advantage   
that its value depends on the representation of the data and on the 
learning algorithm used to solve the tasks. This makes it possible to 
interpret the bound as a quality measure of the transferred information. 
Therefore, by optimizing the measure we obtain principled algorithms for lifelong learning. 

In Sections \ref{parameter-transfer} and \ref{representation-transfer}
we demonstrate this process in two cases: assuming that the solutions 
to all tasks can be represented by a single parameter vector plus small 
task-specific perturbation~\citep{Evgeniou}, we obtain an algorithm that 
resembles previously proposed methods for regularizing the weight vectors 
of future tasks using linear combinations of weight vectors of previous 
tasks, such as~\citep{Yang,Aytar}.

An alternative assumption is that the solution vectors to tasks can differ 
significantly, but that they all lie in a common feature subspace of low 
dimension.
In this setting, our bound provides an algorithm in which the observed tasks 
are used to identify the most promising subspace of features, such that learning 
for future tasks needs to take place only within the reduced feature space. 
This procedure is related to existing methods for representation and 
dictionary learning, e.g.\ \citep{Argyriou,Kumar}, which have also successfully 
been applied in the lifelong learning setting~\citep{Ruvolo}. 
In a case of linear regression,~\citet{journals/ml/Maurer09} used this assumption 
to prove a generalization bound in the PAC framework, 
by using the concept of environment of tasks from~\citet{Baxter} and Rademacher complexity.   
Similar results were obtained in~\cite{maurer13} in the case of sparsity constraints.
\paragraph{The PAC-Bayesian framework.}
For the convenience of readers who are not familiar with the PAC-Bayesian 
framework, we introduce the most relevant concepts from the literature here. 
For more details, see~\citep{langford2005, Seeger, catoni2007pac}.

PAC-Bayesian theory studies the properties of randomized predictors, called 
\emph{Gibbs predictors}.
Formally, let $\mathcal{X}$ be an input set, $\mathcal{Y}$ an output set and 
$\H \subset \{ h: \mathcal{X}\to\mathcal{Y}\}$ a set of 
prediction functions (hypotheses). 
For any probability distribution $P$ over $\H$, the Gibbs 
predictor associated with $P$ is the stochastic predictor that 
for any $x\in\mathcal{X}$ randomly samples a hypothesis $h\sim P$ 
and then returns $h(x)$. 

Assume that we are given a set $S=\{(x_1,y_1),\dots,(x_m,y_m)\}$ of 
i.i.d.\ samples from an unknown probability distribution $D$ over 
$\mathcal{X}\times\mathcal{Y}$. 
For any loss function, $\ell:\mathcal{Y}\times \mathcal{Y}\rightarrow [0,1]$, let $\er(Q)$ 
denote the expected loss of the Gibbs classifier associated with $Q$, 
i.e.\ $\er(Q) = \E_{h\sim Q}\E_{(x,y)\sim D} \ell(h(x),y)$, and let $\eer(Q)$ 
denote the expected empirical loss, i.e.\ $\eer(Q) = \E_{h\sim Q}\frac1m\sum_{i=1}^m \ell(h(x_i),y_i)$.
It is then possible to prove generalization bounds such as the following:
with probability at least $1-\delta$ (over the sampling of $S$) we 
have for all distributions $Q$ over $H$~\cite{mcallester1999pac}
\begin{equation}
\er(Q)\! \leq \eer(Q)\!+\!\sqrt{\frac{\KL(Q||P) + \log\frac{1}{\delta} +\log m + 2}{2m-1}},
\label{eq:PACBayes}
\end{equation}
where $P$ is a reference or \emph{prior} distribution over $\H$ that must be chosen before observing the samples $S$, and $\KL(Q||P)$ is the Kullback-Leibler divergence, i.e.\ a measure how different $Q$ is from $P$. 
As such, the bound resembled the typical trade-off in regularized risk minimization 
between the training loss and a regularizer~\citep{vapnik1998statistical}.

Inequality~\eqref{eq:PACBayes} is uniform with respect to $Q$, so it holds 
regardless of which $Q$ we choose. In particular, we can choose it after 
seeing $S$, and for this reason $Q$ is typically referred to 
as \emph{posterior} distribution in this context. 

Choosing $Q$ such that it minimizes the right hand side of the bound we obtain a 
Gibbs predictor that can be expected to be a good choice for the learning task 
at hand, since its expected loss is controlled by a hopefully small quantity. 
While the inequality~\eqref{eq:PACBayes} holds regardless the agreement between the data distribution and the prior distribution $P$, the value of the right hand side of the bound strongly depends on the choice of $P$. 
Therefore, one would prefer a prior that allows learning a posterior that is at the same time close to the prior ($KL(Q||P)$ is small) and shows good performance on the training set ($\eer(Q)$ is small).  

\section{PAC-Bayesian Lifelong Learning}\label{sec:bound}
To develop a PAC-Bayesian theory of lifelong learning we adopt the concept of 
a task environment from~\citet{Baxter}. 
We assume an unknown set of possible tasks $T$, all of which share 
the same input space $\mathcal{X}$, output space $\mathcal{Y}$, hypothesis set $H$
and loss function $\ell:\mathcal{Y}\times\mathcal{Y}\to [0,1]$.
The lifelong learning system (which we will call an \emph{agent}) 
observes $n$ tasks $t_1,\dots,t_n$ that are sampled i.i.d.\ from $T$ 
according to some unknown distribution over tasks. 
For each task $t_i$ the agent observes a training set 
$S_i = \{(x_{i1}, y_{i1}),\dots,(x_{im_i},y_{im_i})\}$ that is
sampled i.i.d.\ according to the task's unknown data distribution $D_i$.
To solve individual tasks, the agent makes use of an arbitrary but fixed 
learning algorithm, i.e.\ a deterministic procedure that, given a training 
set $S$ and a form of prior knowledge $P$, outputs a posterior distribution 
$Q=Q(S,P)$ over $\H$.
The agent makes predictions using Gibbs predictor associated with $Q$. 
Staying within a PAC-Bayesian setting we assume that the prior 
knowledge, $P$, is encoded in a probability distribution over $H$. 
For concrete examples of the above setting see 
Sections~\ref{parameter-transfer} and~\ref{representation-transfer}. 

The goal of the agent is to use the information contained in the observed 
tasks to identify prior knowledge that will cause as good as possible 
performance on new (so far unobserved) tasks from the same environment. 
This setting is strictly harder than multi-task learning or domain 
adaptation, since no data for the future tasks to be solved is 
available at the time the agent makes its decision. 
In particular, previously developed techniques for learning priors are
not directly applicable: 
first, note that we cannot use ordinary generalization bounds, such as~\eqref{eq:PACBayes}, 
to identify optimal priors, since they only hold uniformly in $Q$ if the prior 
is chosen independently from the training set. 
\citet{catoni2007pac} derived an expression for the overall "best" prior, 
i.e.\ the distribution resulting in the smallest possible bound value. 
However, it is generally of a non-parametric form and uncomputable 
without full information about the data distribution. 
\citet{parrado2012pac} showed that priors can be learned by splitting 
the available training data into two parts, one for learning a prior, 
one for learning the predictor. This, however, requires training data 
for the task at hand, which is not available in the lifelong setting. 

Our first contribution in this work is the insight that one should 
treat the prior $P$ itself as a random variable. 
Let $\mathcal{P}$ be an initial distribution over all possible 
priors, which we call \emph{hyperprior} in concordance with the 
Bayesian nomenclature.
For learning a prior the agent uses the observed tasks to adjust 
its original hyperprior into a \emph{hyperposterior} distribution 
$\mathcal{Q}$ over the set of priors. 
This randomized setting allows us to follow a PAC-Bayesian path analogous 
to classical results. We will obtain a bound that requires a fixed 
hyperprior $\mathcal{P}$ but that holds uniformly with respect to the 
hyperposterior $\mathcal{Q}$. 
The hyperposterior that minimizes the bound will provide us with the 
most promising distribution from which to obtain priors for future 
tasks.

Formally, the goal of the agent is to find $\mathcal{Q}$ that minimizes 
the expected loss $\er(Q_t)$ of a randomly sampled new task $t$ with 
training set $S_t$ and prior $P$ sampled from $\mathcal{Q}$. 
We write
\begin{equation}
\er(\mathcal{Q})=\mathbf{E}_{(t,S_t)} \mathbf{E}_{P\sim\mathcal{Q}}\ \er(Q_t(S_t,P) ),
\label{exp_risk}
\end{equation}
where $Q_t(S_t,P)$ is the posterior obtained by training the learning algorithm with prior $P$ and training sample $S_t$. We call this quantity the \textit{transfer risk}.

We cannot compute $\er(\mathcal{Q})$ because the distributions over the tasks and the tasks' data are both unknown. 
However, we can approximate it by its empirical counterpart, based on $n$ observed tasks
\begin{equation}
\eer(\mathcal{Q}) = \frac{1}{n}\sum\nolimits_{i=1}^n \mathbf{E}_{P\sim\mathcal{Q}}\  \eer(Q_{i}(S_i,P)),
\label{emp_risk}
\end{equation}
which we call \textit{empirical multi-task risk}. 

Our main result is a theorem that bounds the difference between the two quantities defined above. 
%
%
\begin{theorem} 
For any $\delta > 0$ the following inequality holds with probability 
at least $1-\delta$ (over the training samples $\{S_1,\dots,S_n\}$)
for all hyperposterior distributions $\mathcal{Q}$
\begin{align}
&\er(\mathcal{Q})\leq \eer(\mathcal{Q}) +\frac{1}{\sqrt{n}}\left(\KL(\mathcal{Q}\|\mathcal{P}) +\frac{1}{8}-\log\frac{\delta}{2}\right)
\label{main_result}
\\
&+ \frac{1}{n\sqrt{\bar{m}}}\KL((\mathcal{Q},Q^n)\|(\mathcal{P},P^n))
+ \frac{1}{\sqrt{\bar{m}}}\left(\frac{1}{8}-\frac{1}{n}\log\frac{\delta}{2}\right)
\notag\end{align}
where $(\mathcal{Q},Q^n)=\mathcal{Q}\times \prod_{i=1}^n Q_{i}$ denotes the 
distribution in which we first sample $P$ according to $\mathcal{Q}$ and then use it and the 
data $S_i$ to produce a posterior $Q_i$ for each task $t_i$. 
$(\mathcal{P},P^n)=\mathcal{P}\times \prod_{i=1}^n P$ denotes the distribution in 
which we sample $P$ according to $\mathcal{P}$ and use it as a posterior for all tasks. 
$\bar{m}=\left(\frac{1}{n}\sum_{i=1}^n\frac{1}{m_i}\right)^{-1}$ is the harmonic mean of the sample sizes.
\label{theorem}
\end{theorem}
\begin{proof}
To prove Theorem~\ref{theorem} we introduce an intermediate quantity that can be seen as an \textit{expected multi-task risk}
\begin{equation}
\ter(\mathcal{Q}) = \underset{P\sim\mathcal{Q}}{\mathbf{E}}\frac{1}{n}\sum\nolimits_{i=1}^n\underset{h\sim Q_i}{\mathbf{E}}\underset{(x,y)\sim D_i}{\mathbf{E}}\ell(h(x), y).
\end{equation}
First we will bound the uncertainty on the task environment level by bounding the difference 
between transfer error, $\er(\mathcal{Q})$, and expected multi-task error, $\ter(\mathcal{Q})$. 
Then we will bound the uncertainty within observed tasks by bounding the difference between 
expected multi-task error, $\ter(\mathcal{Q})$, and its empirical approximation, $\eer(\mathcal{Q})$.
Our main tool in both cases will be the following lemma.
\begin{lemma}
\label{lemma:main_lemma}
Let $f$ be a random variable taking values in $A$ and 
let $X_1,\dots,X_l$ be $l$ independent random variables with each 
$X_k$ distributed according to $\mu_k$ over the set $A_k$. 
For functions $g_k: A\times A_k\rightarrow [a_k,b_k], \; k=1\dots l$, 
let $\xi_k(f) = \E_{X_k\sim\mu_k}g_k(f,X_k)$ for any fixed value of $f$.
Then for any fixed distribution $\pi$ on $A$ and any $\lambda, \delta>0$ 
the following inequality holds with probability at least $1-\delta$ (over sampling $X_1,\dots,X_l$) 
for all distributions $\rho$ over $A$  
\begin{align}
\notag
&\;\;\;\underset{f\sim\rho}{\mathbf{E}}\sum\nolimits_{k=1}^l\xi_k(f)-\underset{f\sim\rho}{\mathbf{E}}\sum\nolimits_{k=1}^lg_k(h,X_k)\leq
\\
\frac{1}{\lambda}&\left(KL(\rho||\pi) + \frac{\lambda^2}{8}\sum\nolimits_{k=1}^l(b_k-a_k)^2 - \log\delta\right). 
\end{align}
\end{lemma}
For the proof of this lemma, see the Appendix~\ref{appendix}.

In order to bound the difference between $\er(\mathcal{Q})$ and $\ter(\mathcal{Q})$ 
we treat each task $t$ with the corresponding training sample $S_t$ as a random 
variable and apply Lemma~\ref{lemma:main_lemma}. 
Formally, we set $\rho = \mathcal{Q}$, $\pi=\mathcal{P}$, $X_k=(t_k,S_k)$, $l=n$, $f = P$ 
and $g_k(f, X_k)=\frac{1}{n}\underset{h\sim Q_k}{\mathbf{E}}\underset{(x,y)\sim D_k}{\mathbf{E}}l(h(x),y)$
and apply Lemma~\ref{lemma:main_lemma} with $\lambda = \sqrt{n}$. 
Since $a_k = 0$ and  $b_k = \frac{1}{n}$ we obtain with probability at least $1-\delta/2$ that for all $\mathcal{Q}$
\begin{equation}
\er(\mathcal{Q})\leq\ter(\mathcal{Q})+\frac{1}{\sqrt{n}}\left(\KL(\mathcal{Q}||\mathcal{P}) +\frac{1}{8}-\log\frac{\delta}{2}\right).
\label{bound1}
\end{equation}
To bound the difference between $\ter(\mathcal{Q})$ and $\eer(\mathcal{Q})$ 
we apply Lemma~\ref{lemma:main_lemma} to the union of all training samples 
$S'=\bigcup_{i=1}^nS_i$. We set $\rho = (\mathcal{Q}, Q^n)$, $\pi = (\mathcal{P},P^n)$, $X_k = (x_{ij}, y_{ij})$, $l=\sum m_i$, $f=(P, h_1,\dots,h_n)$ 
and $g_k(f, X_k) = \frac{1}{nm_i}\ell(h_i(x_{ij}),y_{ij})$. 
In this setting $a_k=0$ and $b_k=1/(nm_i)$, Lemma~\ref{lemma:main_lemma} with $\lambda = n\sqrt{\bar{m}}$ yields that with probability at least $1-\delta/2$ for all $\mathcal{Q}$
\begin{align}
\ter(\mathcal{Q})&\leq\eer(\mathcal{Q}) + \frac{1}{n\sqrt{\bar{m}}}\KL((\mathcal{Q},Q^n)||(\mathcal{P},P^n))
\notag
\\
&\qquad\qquad+ \frac{1}{8\sqrt{\bar{m}}}-\frac{1}{n\sqrt{\bar{m}}}\log\frac{\delta}{2}.
\label{bound2}
\end{align}
Now~\eqref{main_result} follows by a union bound from~\eqref{bound1} and \eqref{bound2}.
\end{proof}
To get a better understanding of Theorem~\ref{theorem}, we rewrite \eqref{main_result} in the following way:
\begin{align}
&\er(\mathcal{Q}) \leq \eer(\mathcal{Q})+\left(\frac{1}{\sqrt{n}}  + \frac{1}{n\sqrt{\bar{m}}}\right)\KL(\mathcal{Q}\|\mathcal{P})
\label{main_result_rw}
\\
&\ +\frac{1}{n\sqrt{\bar{m}}}\sum_{i=1}^n\underset{P\sim\mathcal{Q}}{\mathbf{E}}\KL(Q_i(S_i,P)\|P) + \const(n,\bar m,\delta).
\notag
\end{align}
We see that the bound contains two types of complexity terms
that correspond to two levels of our model: $\KL(\mathcal{Q}\|\mathcal{P})$ belongs to the level of 
task environment in general, while each $\KL(Q_i(S_i,P)\|P)$ corresponds specifically to the $i$-th task. 

To better understand their roles, we look at the following limit cases: 
when the agent has access to sufficiently 
many tasks ($n\to\infty$) but tasks come with a finite amount of data ($\bar m$ is finite), the 
first complexity term converges to $0$ as $1/\sqrt{n}$. The second complexity term converges to an average $\KL$-divergence 
over tasks and may therefore remain non-zero. 
This means that observing many task gives the agent full knowledge about 
the task environment, but it cannot overcome the uncertainty within each task. 
In the opposite case, if the agent observes unlimited data for each tasks, but only for a 
finite number of tasks ($\bar{m}\to\infty$, $n$ is finite), 
the second complexity term converges to $0$ as $1/\sqrt{\bar{m}}$, while the first one does not, so 
there is still uncertainty on the task environment level. 
Only when both comes together, sufficiently many tasks and sufficient amounts of data per task, 
it is guaranteed that the empirical multi-task risk $\eer(\mathcal{Q})$ converges to the transfer 
risk $\er(\mathcal{Q})$.

A second important aspect of Theorem~\ref{theorem} is that the 
bound~\eqref{main_result} consists only of observable quantities. 
Therefore, we can treat it as a quality measure for 
hyperposteriors $\mathcal{Q}$. 
By minimizing it, we obtain a hyperposterior distribution 
over priors that is adjusted to the particular environment of learning tasks.
Since the bound holds uniformly with respect 
to $\mathcal{Q}$, the guarantees of Theorem~\ref{theorem} also 
hold for the resulting learned hyperposterior, so we can expect 
priors sampled according to the learned hyperposterior to work
well even for future tasks.

In the following sections, we discuss two instantiations of this procedure and show how they relate to previous work on transfer learning.  
\subsection{Parameter Transfer} 
\label{parameter-transfer}
Let $\mathcal{X}\subset\mathbb{R}^d$ and $H$ be a set of linear predictors: 
$h(x)=\langle w, x\rangle$ if $\mathcal{Y}=\mathbb{R}$ or $h(x)=\sign\langle w, x\rangle$ 
if $\mathcal{Y}=\{-1,1\}$, where $w\in \mathbb{R}^d$ is a weight vector. 
One of the common assumptions in multitask or lifelong learning is 
that the weight vectors for different tasks are only minor 
variations of an unknown prototypical vector~\citep{Evgeniou}. 
It can be captured by regularizing the distance to this vector~\citep{Aytar,Yang}:
\begin{equation}
\hat{w} = \arg\min_w \Big(\|w - w_{pr}\|^2 + \frac{C}{m}\sum_{j=1}^m (y_j\!- \langle w, x_j\rangle)^2\Big),
\label{ASVM}
\end{equation} 
where $w_{pr}$ is some function of weight vectors of previously observed 
tasks, e.g.\ just their average, $w_{pr}=\frac{1}{n}\sum_{i=1}^n w_i$.

Theorem~\ref{theorem} allows us to learn an ``optimal" $w_{pr}$ from 
the data, instead of fixing the rule for computing it. 
For this, we choose $P=\mathcal{N}(w_P, I_d)$ and $Q=\mathcal{N}(w_Q, I_d)$, 
i.e.\ unit variance normal distributions with means $w_P$ and $w_Q$, 
respectively. 
The mean $w_P$ is a random variable distributed first according to the 
hyperprior distribution, $\mathcal{P}$, which we set as $\mathcal{N}(0,\sigma I_d)$ 
and later according to the hyperposterior, $\mathcal{Q}$, which we model as $\mathcal{Q}=\mathcal{N}(w_\mathcal{Q}, I_d)$. 
The task of the learning consists of identifying the best $w_\mathcal{Q}$.

As underlying learning algorithm we use Equation~\eqref{ASVM} with regularizer centered at a prior vector.
For any $w_P$ and training set $S=\{(x_i,y_i)_{i=1,\dots,m}\}$ the
posterior, $Q(S,P)=\mathcal{N}(w_Q, I_d)$, is given by 
\begin{equation}
w_Q = \argmin\!\Big(\|w \!-\! w_P\|^2\! +\! \frac{C}{m}\!\sum_{j=1}^m (y_j\! -\!\langle w,x_j\rangle)^2\Big).
\label{learning_alg1}
\end{equation}
This has the closed form solution 
\begin{equation}
w_Q = \big(\frac{m}{C}I_d\! +\! XX^\top \big)^{-1}\big(\frac{m}{C}w_P\! +\! XY\big) = Aw_P\! + \!b,
\label{learning_alg1_solution}
\end{equation}
where $X$ is the matrix with columns $x_1,\dots,x_m$, $Y$ is a column of labels $(y_1,\dots,y_m)^\top$, $A=\left( I_d + \frac{C}{m} XX^\top\right)^{-1}$ and $b = \frac{C}{m}AXY$.

Computing the complexity terms from \eqref{main_result_rw} we obtain
\begin{align}
&\KL(\mathcal{Q}\|\mathcal{P}) = \frac{\|w_\mathcal{Q}\|^2}{2\sigma}+\frac{d}{2}\Big(\log\sigma + \frac{1}{\sigma} - 1\Big)
\notag
\;\;\;\text{and}\\
&\underset{P\sim\mathcal{Q}}{\mathbf{E}}\KL(Q_i(S_i,P)\|P) = \underset{w_P\sim\mathcal{Q}}{\mathbf{E}}\frac{\|(A_i-I_d)w_P + b_i\|^2}{2} 
\notag
\\ 
&\quad= \frac{1}{2}\Big(\|(A_i-I_d)w_\mathcal{Q} + b_i\|^2 + \tr(A_i-I_d)^2\Big).
\label{KL2}
\end{align}

We insert Equations~\eqref{KL2} into the inequality \eqref{main_result_rw} and obtain 
\begin{align}
\forall w_{\mathcal{Q}}\quad \er(w_{\mathcal{Q}})\leq 
\eer(w_{\mathcal{Q}})
+ \frac{\sqrt{n\bar{m}} + 1}{2\sigma n\sqrt{\bar{m}}}\|w_\mathcal{Q}\|^2
\notag
\\
+\frac{1}{2n\sqrt{\bar{m}}}\sum_{i=1}^n\|(A_i-I_d)w_\mathcal{Q} + b_i\|^2
+ \const.
\label{gaussian-bound}
\end{align}

The last thing we have to specify is the loss function $\ell$. We consider two options:
first, the binary classification setting with $0/1$ loss: $\ell(y_1, y_2) = \llbracket y_1\neq y_2\rrbracket$. 
In this case the expected empirical error of the Gibbs classifier is given by the following  
expression \cite{Germain:2009, langford2002}
\begin{equation}
\eer(w_\mathcal{Q})\! =\! \frac{1}{n}\!\sum_{i=1}^n\!\frac{1}{m_i}\!\sum_{j=1}^{m_i}\! \Phi\!\left(\!\frac{y_{ij}x_{ij}^\top(A_iw_{\mathcal{Q}} + b_i)}{\sqrt{x_{ij}^\top(I_d\!+\!A_iA_i^\top)x_{ij}}}\!\right)\!,
\label{emp_error_gaus}
\end{equation}
where $\Phi(z) = \frac{1}{2}\big(1 - \Erf(\frac{z}{\sqrt{2}})\big)$ and 
$\Erf(z)=\frac{2}{\sqrt{\pi}}\int_0^z e^{-t^2}dt$ is the Gauss error function.

For a practical algorithm, one typically would prefer a bound on the loss of a deterministic 
classifier rather than of the stochastic Gibbs classifier.
For $0/1$-loss Theorem~\ref{theorem} provides this, since the Gibbs error is at 
most twice smaller than the expected error 
of the classifier defined by $A_iw_{\mathcal{Q}} + b_i$~\citep{mcallester2003simplified, DBLP:journals/jmlr/LavioletteM07}. 
Inserting \eqref{emp_error_gaus} in \eqref{gaussian-bound} 
and multiplying the left hand side by $\frac{1}{2}$ we obtain the following inequality:
\begin{align}
&\forall w_{\mathcal{Q}}\quad \frac{1}{2}\underset{(t,S_t)}{\mathbf{E}}\underset{(x,y)\sim D_t}{\mathbf{E}}[y\neq\sign\langle A_tw_{\mathcal{Q}} + b_t, x\rangle]\leq 
\label{gaussian-0-1-bound}
\\
&\frac{\sqrt{n\bar{m}} + 1}{2\sigma n\sqrt{\bar{m}}}\|w_\mathcal{Q}\|^2
+\frac{1}{2n\sqrt{\bar{m}}}\sum_{i=1}^n\|(A_i-I_d)w_\mathcal{Q} + b_i\|^2
\notag
\\
&+ \frac{1}{n}\sum_{i=1}^n\frac{1}{m_i}\sum_{j=1}^{m_i} \Phi\left(\frac{y_{ij}x_{ij}^\top(A_iw_{\mathcal{Q}} + b_i)}{\sqrt{x_{ij}^\top(I_d+A_iA_i^\top)x_{ij}}}\right)+ \const. 
\notag
\end{align}

For regression tasks, we consider the case of truncated squared loss, 
$\ell(y_1,y_2)=\min\{ (y_1-y_2)^2, 1\}$ (the truncation is necessary to 
fulfill the condition of a bounded loss function). 
Since $\ell(y_1,y_2)\leq(y_1-y_2)^2$, we can substitute $\eer(w_{\mathcal{Q}})$ 
in \eqref{gaussian-bound} by the empirical error of the Gibbs predictor 
with squared loss without violating the inequality. 
This error differs from the error of the predictor that is defined 
by $A_iw_{\mathcal{Q}} + b_i$ only by a constant that does not depend 
on $w_\mathcal{Q}$. 
An elementary calculation shows that for truncated squared loss $\ell$, 
as in the case of $0/1$ loss, the error of Gibbs predictor is at least 
one half of the expected error of the predictor defined by $A_iw_{\mathcal{Q}}+b_i$. 
Therefore in this case we obtain a result similar to the inequality \eqref{gaussian-0-1-bound}
\begin{align}
&\forall w_{\mathcal{Q}}\quad \frac{1}{2}\underset{(t,S_t)}{\mathbf{E}}\underset{(x,y)\sim D_t}{\mathbf{E}}\min\{(y-\langle A_tw_{\mathcal{Q}} + b_t, x\rangle)^2, 1\}\leq 
\notag
\\
&\frac{\sqrt{n\bar{m}} + 1}{2\sigma n\sqrt{\bar{m}}}\|w_\mathcal{Q}\|^2
+\frac{1}{2n\sqrt{\bar{m}}}\sum_{i=1}^n\|(A_i-I_d)w_\mathcal{Q} + b_i\|^2
\notag
\\
&+ \frac{1}{n}\sum_{i=1}^n\frac{1}{m_i}\!\sum_{j=1}^{m_i}(y_{ij} \!-\! \langle A_iw_{\mathcal{Q}}\! +\! b_i, x_{ij}\rangle)^2+ \const. 
\label{gaussian-squared-loss-bound}
\end{align}

Minimizing the right hand side of \eqref{gaussian-0-1-bound} 
or \eqref{gaussian-squared-loss-bound} with respect to $w_{\mathcal{Q}}$, 
we obtain a data-dependent hyperposterior that induces prior distributions 
that are adjusted optimally (in the sense of the bound) to the task environment.

\subsection{Representation Transfer}
\label{representation-transfer}
A second assumption commonly made in multitask or lifelong learning 
is that the weight vectors for all tasks lie in low-dimensional subspace. 
Theorem~\ref{theorem} also allows us to learn such a subspace in a principled way.

We again assume that $\mathcal{X}\subset\mathbb{R}^d$ and 
$H$ is a set of linear predictors. 
We represent $k$-dimensional subspaces of $\mathbb{R}^d$ by $d\times k$ 
matrices with orthogonal columns, i.e.\ elements of the Stiefel 
manifold $V_{d,k}$. 
As hyperprior, we want all subspaces to be equally likely, so we 
set $\mathcal{P}$ to the uniform distribution over $V_{d,k}$~\cite{Downs}
\begin{align}
p_{\mathcal{P}}(B) &= \frac{1}{C_0} \:\:\: \text{for any}\ B \in V_{d,k},
\intertext{where $C_0 = {}_0F_1(\frac{1}{2}d,0)$. 
As hyperposterior, $\mathcal{Q}$, we want a distribution that concentrates 
its probability mass around a specific subspace, $M$. 
We choose a special case of Langevin distribution, $D(I_k,M)$, 
}
p_{\mathcal{Q}}(B) &= \frac{1}{C_1}\exp(\tr(M^\top B)) \:\:\: \text{for any}\ B \in V_{d,k},
\label{langevin}
\end{align}
where $C_1 = {}_0F_1(\frac{1}{2}d,\frac{1}{4}M^\top M)$. 
The only free parameter is $M\in V_{d,k}$, i.e.\ a $d\times k$ matrix with $M^\top M=I_k$
that represents the "most promising subspace". 
Equation~\eqref{langevin} can be interpreted as an analog of the Gaussian distribution on $V_{d,k}$,
with mode $M$ and unit variance. 
In the special case of $k=1$, it reduces to the better known Von Mises distribution
on the unit circle~\cite{Downs}.

As in the previous section we use Gaussian distributions for prior and posterior, but defined only 
within the subspaces sampled from $\mathcal{P}$ or $\mathcal{Q}$. 
For the prior, $P$, we choose a Gaussian with zero mean 
and variance $\sigma I_k$. The posterior, $Q$, is a shifted Gaussian with variance $\sigma I_k$ and mean $w_Q$ in the same subspace.
As in the previous section we use ridge regression as learning algorithm, but again only within 
the subspace determined by the prior, 
\begin{equation}
w_Q = \argmin_w\!\left(\!\|w\|^2\!+\!\frac{C}{m}\sum_{i=1}^{m}(y_{i}\!-\! \langle w, B^\top x_{i}\rangle)^2\!\right)\!,
\end{equation}
where $B$ is the matrix representing the subspace, such that $B^\top x$ 
is the projected representation of the training data in this subspace. 

To obtain an objective function for learning $M$, we first compute the complexity terms 
in the bound~\eqref{main_result_rw}. 
$\KL(\mathcal{Q}\|\mathcal{P})$ is a constant independent of $M$: 
$\mathcal{P}$ is uniform, so $\KL(\mathcal{Q}\|\mathcal{P})$ depends only on 
the differential entropy of $\mathcal{Q}$. This itself is a constant independent 
of the parameter matrix\footnote{For any $M\in V_{d,k}$ there exits an orthogonal matrix $L\in \mathbb{R}^{d\times d}$ 
such that $LM = J = \{\delta_{ij}\}\in \mathbb{R}^{d\times k}$. 
Therefore if $B\sim D(I_k, M)$, than $LB\sim D(I_k, LM)=D(I_k, J)$. So, the entropy of $D(I_k, M)$ 
is equal to the entropy of $D(I_k, J)$ for any $M$.}. 
Furthermore, we have $\KL(Q_i(S_i,P)\|P)=\frac{1}{2\sigma}\|w_i(B)\|^2$, where $B$ is the representation of the selected subspace.
In combination, we get the following bound 
\begin{align}
\er(M) &\leq \eer(M) + \frac{1}{2\sigma n\sqrt{\bar{m}}}\sum_{i=1}^n\underset{B\sim D(I_k,M)}{\mathbf{E}}\|w_i(B)\|^2
\notag 
\\ + \const &= \frac{1}{n}\sum_{i=1}^n \E_{B\sim\mathcal{Q}}\left\{ \eer( w_i(B) ) + \frac{1}{2\sigma\sqrt{\bar m}} \| w_i(B) \|^2  \right\}, 
\label{obj_lang}
\notag
\\&\qquad\qquad + \const.
\end{align}
where $w_i(B) = \frac{C}{m_i}\left(I_k + \frac{C}{m_i} B^\top X_iX_i^\top B\right)^{-1}\!\!B^\top X_iY_i$.
%
We see that a representation, $M$, can be 
considered promising for future tasks, if itself as well as the subspaces close to it allow classification 
with small loss and small weight vector norm (i.e.\ large margin) for all observed tasks. 

\section{Experiments}
In this section, we demonstrate how learning priors distributions by minimizing 
the bounds~\eqref{gaussian-0-1-bound}, \eqref{gaussian-squared-loss-bound} and \eqref{obj_lang} 
can improve prediction performance in real prediction tasks. 
To position our results with respect to previous work on parameter and representation transfer, 
we compare to adaptive ridge regression (ARR), i.e.\ Equation~\eqref{ASVM} the prior $w_{pr}$ set to the average of the weight vectors from the observed tasks,
and with the ELLA algorithm~\cite{Ruvolo} that learns a subspace representation using structured
sparsity constraints, also with squared loss.
We also report results for ordinary ridge regression without any knowledge transfer. 

We perform experiments on three public datasets:
 
\textit{Land Mine Detection}~\citep{Xue:2007}. This dataset consists of 14820 data points. 
For each data point there are 9 features extracted from radar images and a binary label $0$ or $1$ corresponding to landmine or clutter. 
We also add a bias term, resulting in $d=10$ features.  
Data points are collected from 29 geographical regions 
and we treat each region as a binary classification task.

\textit{London School Data.} This is a regression dataset, containing exam scores of 15362 students from 139 schools. 
Each student is described by 4 school-specific, 3 student-specific features and a year of examination. 
We use the same procedure as in \citep{Argyriou, Kumar, Ruvolo} to encode them in a set of binary features.
We also add a bias term, so the final data dimensionality is $d=28$. Each school constitutes a task.

\textit{Animals with Attributes Dataset}~\citep{lampert-tpami2013}. This dataset contains 30475 images from 50 classes. 
Each image comes with a 2000-dimensional feature vector, that we reduced to 100 dimensions using PCA. 
We $l_2$-normalize the resulting feature vectors and add a bias term. We select the largest 
class, \emph{collie}, and form 49 binary classification tasks, each of them is a classification 
of \emph{collie} versus one of the remaining classes. 
For each task we use $2\%$ of the data (approximately 20 images) available for \emph{collie} class  
and the same amount of images from the another task, such that data between different tasks does not 
overlap.  
\subsection{Parameter Transfer}
\label{PL-G}
We first perform experiments on prior learning in the setup of parameter transfer, as described in Section~\ref{parameter-transfer},
calling the resulting algorithm \emph{Prior Learning with Gaussian hyperprior (PL-G)}.
For the classification tasks (Landmine and Animals), we optimize the bound~\eqref{gaussian-0-1-bound}. 
To do so we replace $\Phi$ by its convex relaxation, $\Phi_{\textit{cvx}}(z) = \frac{1}{2} - \frac{z}{\sqrt{2\pi}}$,
if $z\leq 0$ and $\Phi_{\textit{cvx}}(z)=\Phi(z)$ otherwise, and use the conjugate gradient method for finding the minimum. 

For the regression tasks (Schools) we first divide labels (examination scores) by their maximum value. 
This allows us to assume that the squared loss will not exceed $1$. We optimize ~\eqref{gaussian-squared-loss-bound}, 
and due to the squared loss, the problem has a closed form solution:
\begin{align}
w_\mathcal{Q} =& -\Big(D+\frac{\sqrt{n\bar{m}}+1}{\sigma n\sqrt{\bar{m}}}I_d 
\\
&+ \frac{1}{n\sqrt{\bar{m}}}\sum_{i=1}^n A_i'^\top A_i'\Big)^{-1}
\notag
\Big(c+\frac{1}{n\sqrt{\bar{m}}}\sum_{i=1}^n A_i'^\top b_i\Big),\\
\label{gaus_squared_solution}
\text{where}& \;A_i'\! = \!A_i-I_d,\: D=\frac{2}{n}\sum_{i=1}^n\frac{1}{m_i}A_i^\top X_iX_i^\top A_i,
\\
c^\top =& \frac{2}{n}\sum_{i=1}^n\frac{1}{m_i}\Big(\frac{C}{m_i} Y_i^\top X_i^\top A_i^\top X_iX_i^\top A_i-Y_i^\top X_i^\top A_i\Big).
\notag
\end{align}
To make results comparable with the baseline algorithms, we report the squared error 
multiplied by the squared value of the maximum examination score. 
\begin{figure*}[t]
\centering
\subfigure[Landmines -- Parameter Transfer]{
\label{subfig:landmine_gaussian}
\includegraphics[width=.3\textwidth]{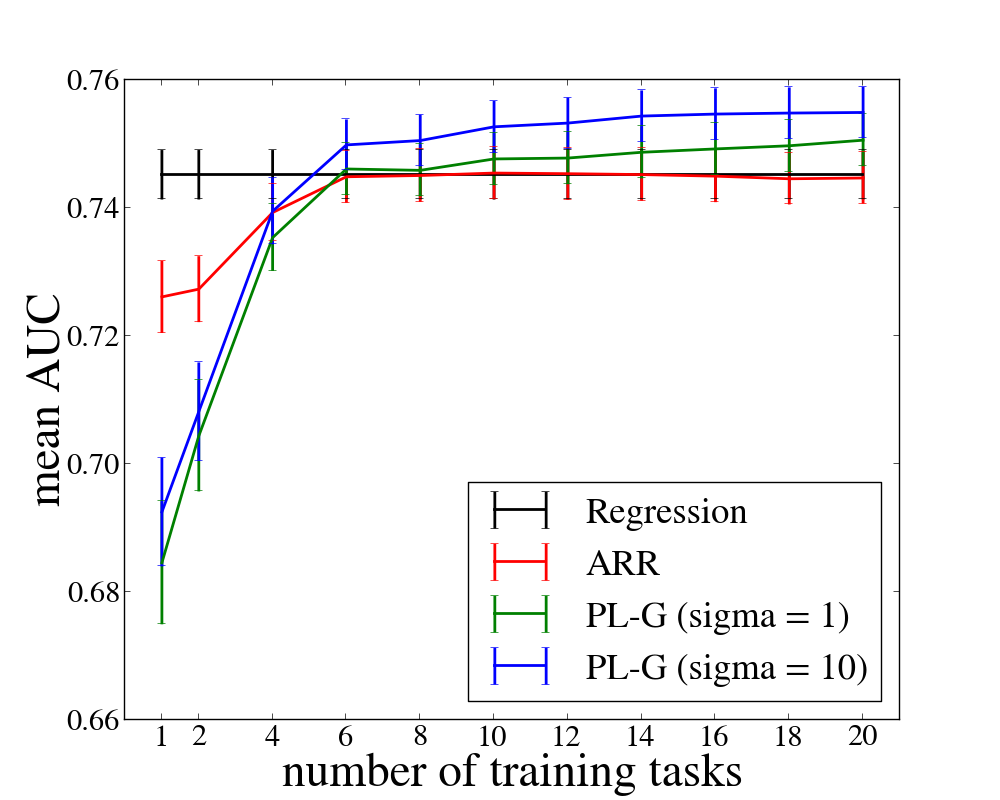}
}
\subfigure[Schools -- Parameter Transfer]{
\label{subfig:schools_gaussian}
\includegraphics[width=.3\textwidth]{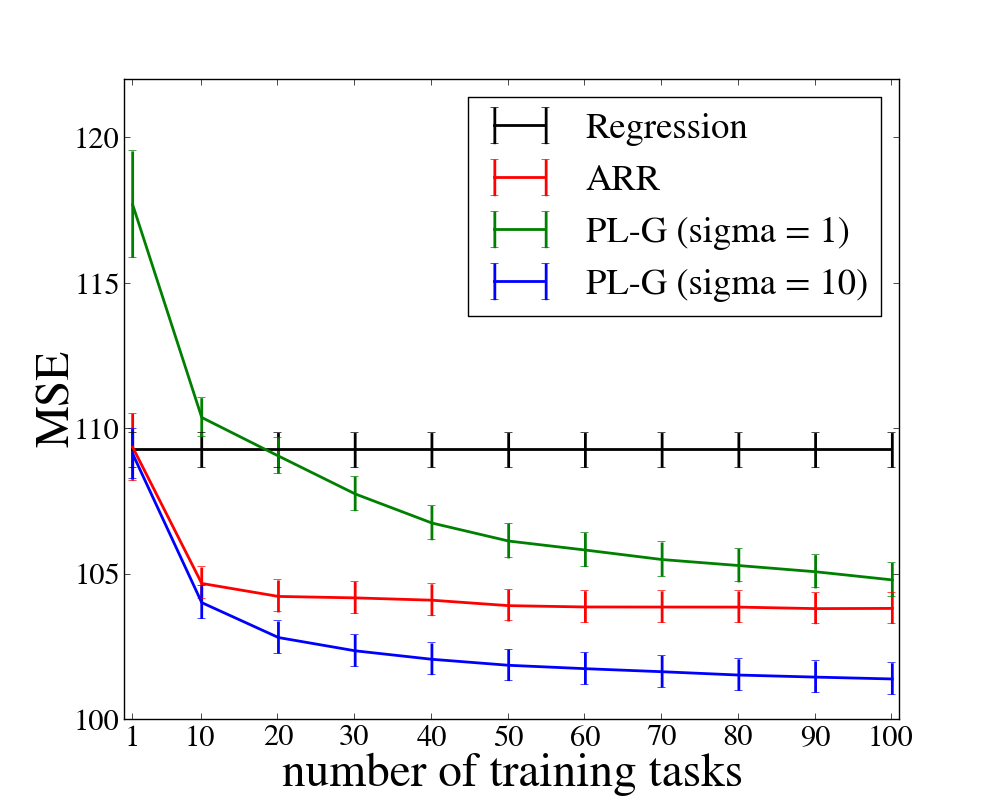}
}
\subfigure[Animals -- Parameter Transfer]{
\label{subfig:animals_gaussian}
\includegraphics[width=.3\textwidth]{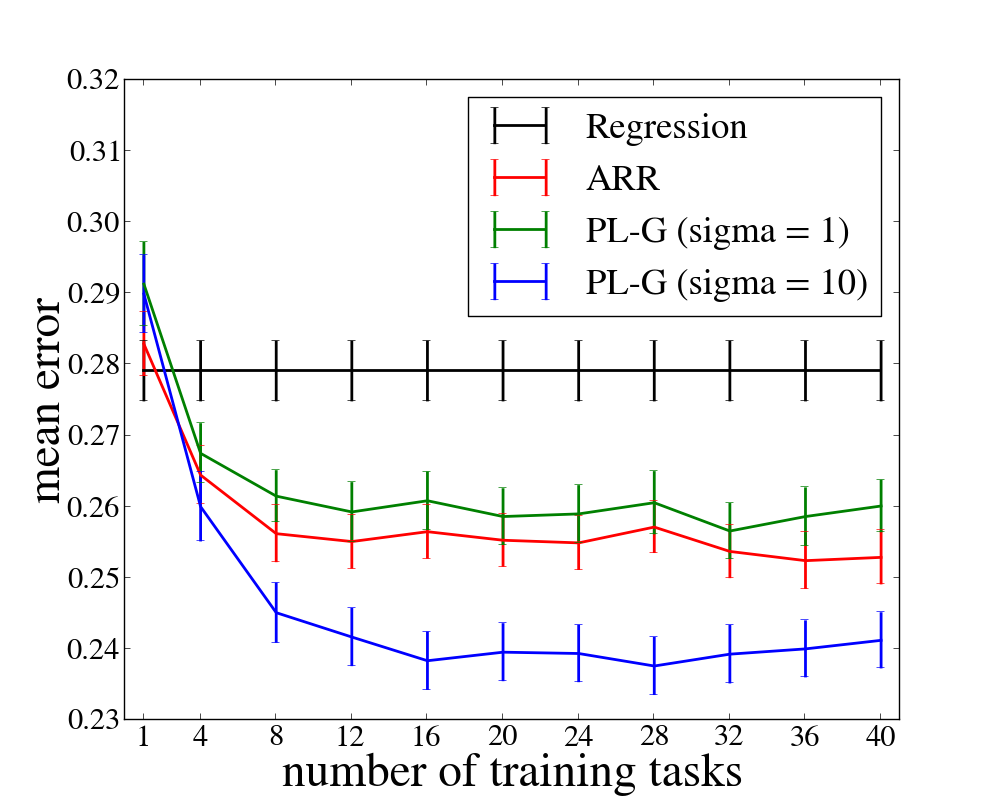}
}
\subfigure[Landmines -- Representation Transfer]{
\label{subfig:landmine_langevin}
\includegraphics[width=.3\textwidth]{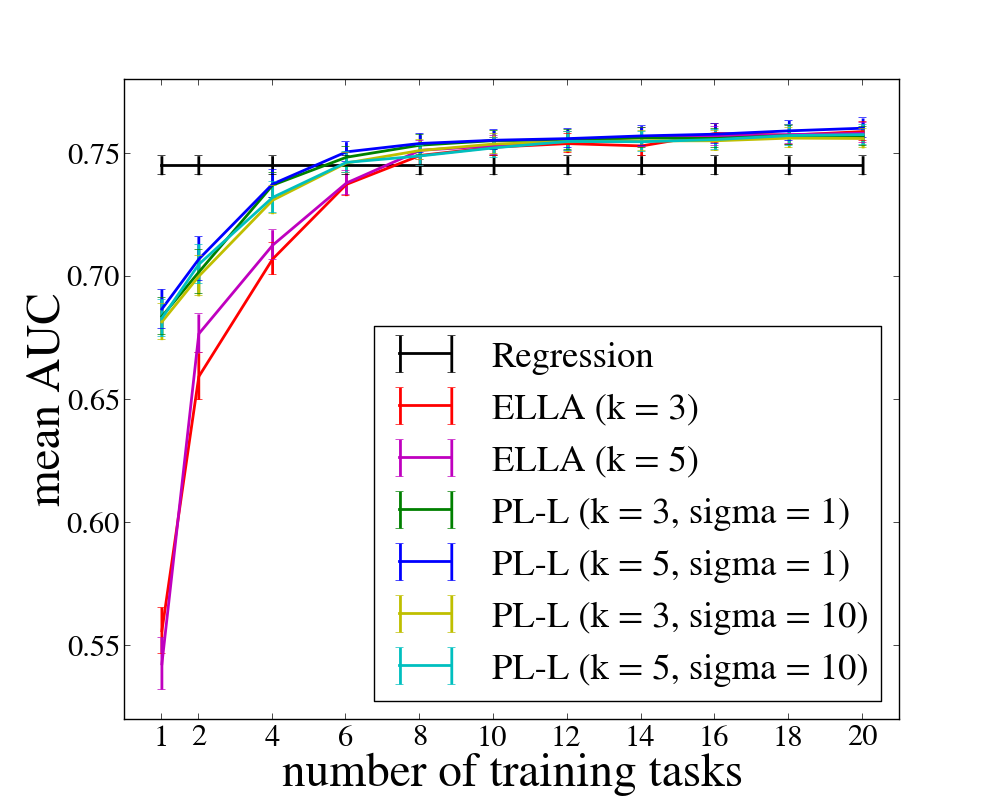}
}
\subfigure[Schools -- Representation Transfer]{
\label{subfig:schools_langevin}
\includegraphics[width=.3\textwidth]{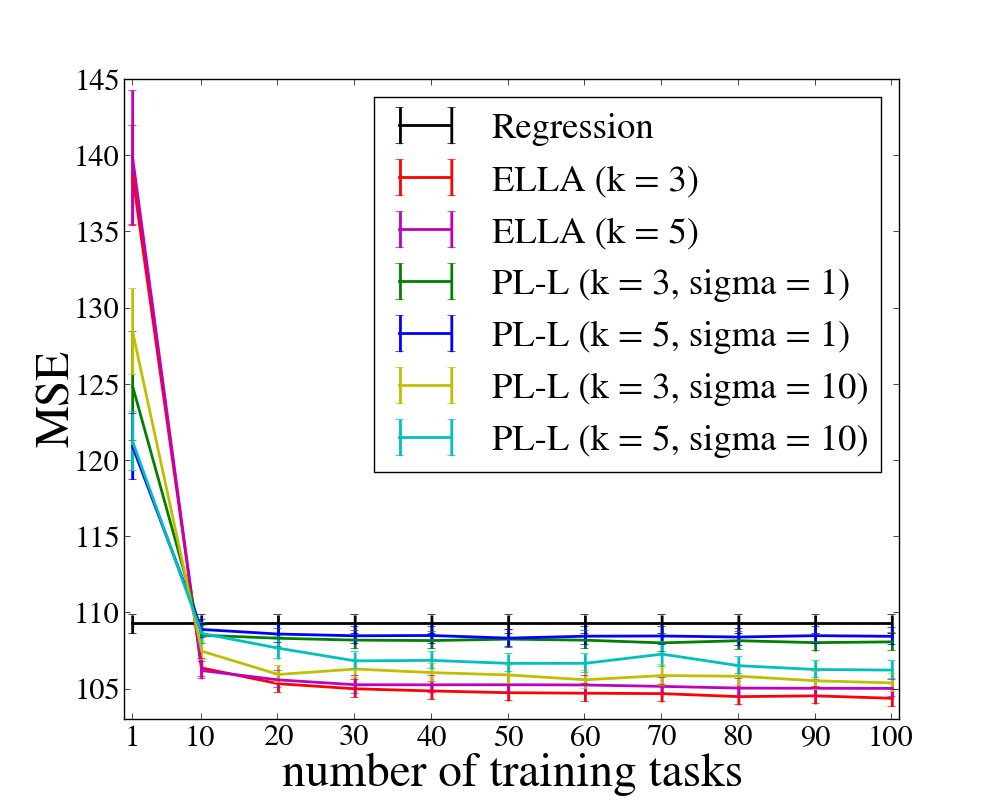}
}
\caption{Results of the experiments on three datasets. For details, see Sections~\ref{PL-G} and \ref{PL-L}}
\label{fig:results}
\end{figure*}
\subsection{Representation Transfer}
\label{PL-L}
In a second set of experiments, we implement the idea of representation transfer from 
Section~\ref{representation-transfer}, calling the algorithm \emph{Prior Learning with Langevin hyperprior (PL-L)}.

As in the case of parameter transfer, we use $0/1$ loss to measure the quality
in classification tasks.
For the regression task we apply the same scaling procedure as discussed in Section \ref{PL-G} and use truncated squared loss. 
Both of these loss functions can be upper-bounded by the standard squared loss, which we do 
to obtain tractable expressions for the right hand side of the Inequality~\eqref{obj_lang}.
To be able to optimize the expression \eqref{obj_lang} numerically, we approximate it 
by replacing all expectations over $\mathcal{Q}$ by their values at its mode, $M$. 
Furthermore, we replace the error of any Gibbs predictor by the error of the 
deterministic predictor defined by the mode of the posterior distribution, $w_i(M)$.
The result is a quadratic optimization problem over the Stiefel manifold, which we solve using \emph{gradient descent with curvilinear search}~\cite{Wen}.

\subsection{Evaluation procedure}
To get reliable estimates of the transfer risk, we repeat the 
following experimental procedure 100 times for each dataset 
and calculate the mean prediction errors and standard errors of the mean.

In each experiment, we set aside a subset of tasks as unobserved 
(9 in Landmines, 39 in Schools, 9 in Animals). These are not used 
during any part of training, but only to evaluate the methods on 
"future" tasks.
Of the remaining tasks we use different fractions to measure 
the effect of a different number of observed tasks. 
The algorithms described in Section~\ref{PL-G} and~\ref{PL-L} and ARR have 
one free parameter, the regularization strength $C\in\{10^{-3},\dots,10^{3}\}$. 
We choose this using 3-fold cross-validation in the following way. 
We split the data of each task into three parts: 
we use the first third of all tasks jointly to learn a prior. 
To evaluate this prior, we then train individual predictors 
using the second part of the data, and test their quality 
on the third part. 
For the ELLA algorithm, we use the same procedure to set the 
regularization strength $\mu$, the remaining parameters we 
leave at their default values. 
For the baseline, we set the regularization using ordinary 
3-fold cross-validation.

\subsection{Results}
\label{sec:results}
The results of the experiments on all three datasets are shown in Figure~\ref{fig:results}. 
Since classes in Landmine dataset are unbalanced, for this problem we report the value 
of area under the ROC curve (AUC, bigger value means better prediction). 
Tasks in the Animals dataset are balanced, so for them we report the standard mean $0/1$ error. 
Since the dataset was too large for the subspace methods, we only report results for the 
parameter transfer techniques.
For the experiment on Schools dataset we report the mean squared error (MSE, smaller values 
mean better prediction). 

As a first observation, Figure~\ref{fig:results} confirms the findings of 
previous work that better prediction can be achieved by transferring information from related tasks. 
Overall, it shows that PL-G and PL-L are comparable to the existing, manually designed, techniques. 
Given sufficiently many tasks, they are able to improve the prediction accuracy over the baseline. 
As an illustration of the hyperprior concept, we show results for PL-G with two different 
values for the Gaussian hyperprior variance 
(Figures~\ref{subfig:landmine_gaussian}, \ref{subfig:schools_gaussian}, \ref{subfig:animals_gaussian}). 
For $\sigma=1$, the adaption pursues in a very conservative way and many tasks are needed to 
find a reliable hyperposterior. 
With $\sigma=10$, convergence is faster, and PL-G achieves results comparable with ARR or even slightly better.
For practical tasks, the hyperprior should possibly be chosen by model selection.

The results for representation transfer (Figures~\ref{subfig:landmine_langevin} and \ref{subfig:schools_langevin}) 
show that the improvements achieved by PL-L are comparable to the ELLA algorithm.
As in the case of parameter transfer, we show results for two different values of the Gaussian prior variance: $\sigma = 1$ and $\sigma = 10$.
While for the Landmine dataset (Figure~\ref{subfig:landmine_langevin}), there is no significant difference in the performance 
for different values of parameters $k$ and $\sigma$, for the Schools dataset (Figure~\ref{subfig:schools_langevin}) the choice of these parameters plays a bigger role. 
We see that the improvements of PL-L with $\sigma =10$ are almost the same as the one achieved by ELLA, while for $\sigma = 1$ they are smaller.      
This might be the effect of too strict hyperparameters that cause the method to be more conservative than necessary. 
Another possible reason for the difference in accuracy is that ELLA makes additional sparsity assumption, which PL-L does not.
\section{Conclusion}
In this work we studied lifelong learning from a theoretical perspective.
Our main result is a generalization bound in a PAC-Bayesian framework (Theorem~\ref{theorem}). 
On the one hand, the bound is very general, allowing us to recover two existing 
principles for transfer learning as special cases: the transfer of classifier parameters, 
and the transfer of subspaces/representations.
On the other hand, the bound consists only of observable quantities, such that 
it can be used to derive principled algorithms for lifelong learning that achieve
results comparable with existing manually designed methods. 

A further use of the bound we see is in using it to study the implicit assumptions 
of possible learning methods. For example, a method obtained by means of a unimodal 
hyperposterior will require all tasks to be related to each other. 
In future work, we plan to explore the potential of integrating more 
realistic assumptions, such as hierarchical or multi-modal hyperposteriors. 
A second interesting direction will be to relax the condition that tasks are sampled 
i.i.d.\ from an environment, e.g.\ into the direction of learning tasks of continuously
improving difficulty~\cite{bengio2009curriculum}.

\paragraph{Acknowledgements.} 
We thank Shai Ben-David, Olivier Catoni and Emilie Morvant
for helpful discussions. 
This work was in parts funded by the European Research Council 
under the European Union's Seventh Framework Programme (FP7/2007-2013)/ERC 
grant agreement no 308036.

\appendix
\section{Proof of Lemma~\ref{lemma:main_lemma}}
\label{appendix}
In the proof we will make use of \emph{Hoeffding's Lemma}:
\begin{lemma}\cite{Hoeffding:1963}
\label{hoeffding_lemma}
 Let $X$ be a real-valued random variable such that $Pr(X\in[a,b]) = 1$ and let $\xi=\E\{X\}$. 
Then
\begin{equation}
\mathbf{E}\left[e^{\lambda(\xi - X)}\right]\leq e^{\frac{\lambda^2(b-a)^2}{8}}.
\end{equation}
\end{lemma}

We will also require the following property of the Kullback-Leibler divergence 
that holds for any $\lambda>0$ and can be proved by convex duality~\citep{Seeger}:
\begin{equation}
\underset{f\sim Q}{\mathbf{E}} g(f)\leq\frac{1}{\lambda}\left(\KL(Q\|P)+\log\underset{f\sim P}{\mathbf{E}}e^{\lambda g(f)}\right).
\label{KL-inequality}
\end{equation} 

We now prove Lemma~\ref{lemma:main_lemma}. 
First, we apply~\eqref{KL-inequality} to $g(f) = \sum_{k=1}^l\xi_k(f) - \sum_{k=1}^l g_k(f,X_k)$, obtaining
\begin{align}
\notag
\underset{f\sim\rho}{\mathbf{E}}\left(\sum_{k=1}^l\xi_k(f) - \sum_{k=1}^l g_k(f,X_k)\right)\leq
\\
\frac{1}{\lambda}\left(\KL(\rho\|\pi)+\log\underset{f\sim \pi}{\mathbf{E}}e^{\lambda g(f)}\right).
\label{KL}
\end{align}
Note, that
\begin{equation}
e^{\lambda g(f)} = \prod_{k=1}^l\exp(\lambda(\xi_k(f) - g_k(f,X_k))),
\end{equation}
since for any fixed $f$ the factors are independent. This allows us to apply Hoeffding's Lemma~\ref{hoeffding_lemma} to each factor:
\begin{equation}
\underset{X_1\sim\mu_1}{\mathbf{E}}\!\cdots\!\underset{X_l\sim\mu_l}{\mathbf{E}}e^{\lambda g(f)} \leq \exp\Big(\frac{\lambda^2}{8}\sum\nolimits_{k=1}^l(b_k-a_k)^2\Big).
\end{equation}
By taking the expectation over $f\sim\pi$ we obtain
\begin{equation}
\underset{f\sim\pi}{\mathbf{E}}\underset{X_1\sim\mu_1}{\mathbf{E}}\!\cdots\!\underset{X_l\sim\mu_l}{\mathbf{E}}e^{\lambda g(f)}\! \leq\! \exp\Big(\frac{\lambda^2}{8}\sum\nolimits_{k=1}^l(b_k\!-\!a_k)^2\Big).
\end{equation}
Since $\pi$ is fixed and does not depend on $X_1,\dots,X_l$, we can exchange the order of expectations. 
By applying Markov's inequality with respect to expectations over $X_1,\dots,X_l$ we obtain that 
with probability at least $1-\delta$:
\begin{equation}
\log\underset{f\sim\pi}{\mathbf{E}}e^{\lambda g(f)} \leq \frac{\lambda^2}{8}\sum\nolimits_{k=1}^l(b_k-a_k)^2-\log\delta.
\label{lasteq}
\end{equation}
We obtain~\eqref{lemma:main_lemma} by combining~\eqref{lasteq} and \eqref{KL}.

\bibliography{bibfile}
\bibliographystyle{icml2014}

\appendix

\end{document}